\documentclass[runningheads]{llncs}
\usepackage[T1]{fontenc}
\usepackage{graphicx}
\usepackage{booktabs}
\usepackage[misc]{ifsym}
\usepackage{subcaption}
\usepackage{amsmath}
\usepackage{amssymb}
\usepackage{mathtools}
\usepackage{subcaption}
\usepackage[textsize=tiny]{todonotes}
\usepackage{hyperref}
\newcommand{\corr}{(\Letter)}
\usepackage{mwe}

\begin{document}

\title{Assessing Alignment and Stability of Feature Importance Explanations via Weight of Evidence}

\titlerunning{Weight of Evidence}

\author{Eddie Conti$^1$ \corr, Claudio Daka$^{2,4}$, 	
Álvaro Parafita$^1$, Antonio L. Alfeo$^3$, Axel Brando$^1$, Mario G.C.A. Cimino$^5$}

\authorrunning{E. Conti, C. Daka, et al.}

\institute{Barcelona Supercomputing Center, Barcelona, Spain \email{econti@bsc.es}
\email{parafita.alvaro@gmail.com}
\email{axelbrando@gmail.com}
\and
University of Florence, Florence, Italy \email{claudio.daka@unifi.it}
\and
Dept. of Theoretical and Applied Sciences, eCampus University, Novedrate, Italy 
\and
SMARTEST Research Center, eCampus University, Novedrate, Italy 
\email{antonioluca.alfeo@uniecampus.it}
\and
Dept Information Engineering, University of Pisa, Pisa, Italy
\email{mario.cimino@unipi.it}}



\maketitle   

\begin{abstract}
Feature importance Methods (FIMs) are widely used in Explainable AI to interpret model predictions, yet attribution scores alone often provide limited insight into the underlying reasoning process. In this work, we introduce a novel perspective by embedding FIMs within a hypothesis-testing framework based on Weight of Evidence (WoE).
We quantify how strongly the observed evidence supports any given hypothesis on feature importance. The reference hypothesis can stem from domain knowledge, ground truth, or be derived from the FIM itself.
This formulation enables a principled evaluation of FIMs, capturing both their alignment with prior knowledge and their variability. We further provide theoretical results linking WoE to attribution variance. Empirical results show the applicability and flexibility of our strategy analyzing LIME and SHAP explanations in settings with different reference hypotheses.
Overall, our framework offers a complementary tool for assessing FIMs through a contrastive, evidence-based lens.

\end{abstract}

\section{Introduction}
As Isaac Asimov would say “[...] science gathers knowledge faster than society gathers wisdom”, the rapid advancement of machine learning has led to the development of increasingly complex models, often without a corresponding understanding of their internal mechanisms. Explainable AI (XAI) approaches can be employed to address this issue thanks since those provide insights into how information is processed and how predictions are generated within AI black-box models. This is particularly critical in high-stakes domains, where interpretability is essential to ensure trust, accountability, and transparency \cite{esteva2019guide,ethics_ai2}. 

In this context, numerous explanation techniques have been proposed, including feature importance methods, counterfactual explanations, prototypes, and rule-based approaches \cite{islam2021explainable,intrinsic_posthoc,linardatos2020explainable}.
However, in the literature, empirical studies often rely on researchers’ implicit assumptions about what qualifies as a “good explanation”, as notions such as interpretability or explainability are themselves inconsistently defined and lack universally accepted evaluation criteria (e.g., how to ensure transparency) \cite{lipton2018mythos}. As noted by Miller \cite{miller2019explanation}, this creates a tension between how explanations are designed and how they are cognitively perceived. To reduce this gap, recent works draw on cognitive and social sciences to align XAI with human reasoning \cite{jacovi2021contrastive,alshehri2021human,alshehri2023explainable}. These studies show that explanations are inherently \textit{contrastive}, humans explain $P$ relative to an alternative $Q$, and \textit{selective}, focusing on a few salient factors. In XAI, this is often formalized through sparsity or simplicity, namely limiting explanations to a small number of features \cite{kadir2023evaluation}, allowing for more knowledge transfer. In this work, we build on this perspective and adopt the Weight of Evidence (WoE) framework \cite{wod1985weight}, a principled statistical tool for comparing alternative hypotheses. 
We reinterpret Feature Importance\footnote{We may refer to FIM as explainer or explanation method and to the vector of feature importances as attribution vector.} (either coming from domain knowledge, intrinsic mechanisms in the model or the FIM itself) as testable hypotheses, and use WoE to quantify how strongly an stochastic FIM supports them.
We summarize the contributions of this study as:
\begin{itemize}
    \item A formal strategy for adapting WoE to test hypotheses on FIMs. Depending on the reference hypothesis, our framework can: 
    \begin{itemize}
        \item Test alignment between local and global explanations.
        \item Measure faithfulness and ground-truth alignment
        \item Capture the stability of stochastic FIMs.    
    \end{itemize}
    \item A proof showing how low FIM variability leads to $+\infty$ WoE when the reference hypothesis is generated using its mean attribution. 
\end{itemize}

The remainder of the paper is structured as follows: Section~\ref{sec:related} reviews related work; Section~\ref{sec:woe} introduces the WoE framework and its adaptation to stochastic FIMs; Section~\ref{sec:variability} establishes the theoretical connection between attribution variability and WoE positive divergence; Section~\ref{sec:experiments} presents the experimental evaluation across three complementary settings; Section~\ref{sec:limitations} discusses limitations of our proposal; and Section~\ref{sec:conclusion} concludes with future directions.

\section{Related Work}
\label{sec:related}
In the XAI literature several metrics have been established to measure the different properties of a FIM \cite{kadir2023evaluation,agarwal2022openxai}. However, one aspect that is lacking in the field of functional evaluation is a measure of how well a FIM aligns with a reference hypothesis. To answer this question, we use the Weight of Evidence (WoE) framework, which to the best of our knowledge, has been employed by few works in XAI. Among them, Melis et al. \cite{melis2021human} generate explanations sequentially by computing WoE values over progressively refined nested hypotheses. Given an instance $x$ with prediction $y^*$, the authors construct a sequence of hypotheses ${y^*} \subseteq U_T \subseteq \ldots \subseteq U_0 = Y$, where $Y$ denotes the full output space (for instance, all possible classes). By evaluating the WoE associated with each refinement step, they quantify how the available evidence (i.e. the instance $x$) increasingly supports more specific hypotheses. On the other hand, Alshehri et al. \cite{alshehri2023explainable} apply WoE in the context of Goal Recognition (GR), where the objective is to infer an agent’s intended goal from a sequence of observed actions. In this setting, partial trajectories serve as evidence, and WoE quantifies how strongly the observed actions, namely the observed steps made by the agent, support one candidate goal over competing alternatives. The framework naturally fits this scenario, as it explicitly formalizes the comparison between alternative hypotheses. Le et al. \cite{le2025evidencedecisionexploringevaluative} combine concept-based explanations with WoE. They extract interpretable concepts through both supervised and unsupervised methods and compute their WoE with respect to a target hypothesis $h$ corresponding to a specific prediction. Since concept extraction is naturally applied to images, this work presents an extension of WoE from tabular data to images. A similar approach, is presented in \cite{parola2026human}, where the authors use a custom metric to quantify the presence or absence of concepts in an explanation, and compute the WoE w.r.t. a hypothesis defined as a set of target predictions in an image classification setting.

In contrast to prior works, which leverages WoE to generate or structure explanation, we adopt a different approach. Instead of using WoE as an explanatory mechanism, we treat an explanation, either coming from the FIM itself or from domain knowledge, as a testable hypothesis about the relevance of features. In this perspective, the explanation specifies which features are and are not considered to be responsible for the model's decision, and the WoE is employed to quantify how much we can support this claim by analyzing the behavior of the FIM across various runs.

\section{Weight of Evidence}
\label{sec:woe}
The Weight of Evidence (WoE) is a statistical concept from information theory \cite{wod1985weight}, which refers to strength of an evidence $e$ in supporting an hypothesis $h$ in contrast to an alternative hypothesis $h'$. In formal terms, the relation is expressed in log-odds terms as:
\begin{equation} \label{eq:woe_definition}
    woe(h : e) = \log\frac{O(h \mid e)}{O(h)},
    \qquad
    O(h \mid e)=\frac{p(h \mid e)}{p(\bar{h} \mid e)},
    \quad
    O(h)=\frac{p(h)}{p(\bar{h})}.
\end{equation}
Using Bayes' Theorem, equation \eqref{eq:woe_definition} reduces to the so-called evidence likelihood form:
\begin{equation} \label{eq:evidence_likelihood}
woe(h:e) = \log{\frac{p(e\mid h)}{p(e\mid h')}}.
\end{equation}
If $woe > 0$, the evidence $e$ supports $h$ over its alternative $h'$. Combining equations \eqref{eq:woe_definition} and \eqref{eq:evidence_likelihood} via Bayes' theorem yields the identity that will be useful later on:
\begin{equation} \label{eq:woe_identity}
    \log{\frac{p(e\mid h)}{p(e\mid h')}} = \log{\frac{p(h\mid e)}{p(h'\mid e)}}+ \log{\frac{p(h')}{p(h)}}.
\end{equation}

Equation \eqref{eq:woe_identity} connects the WoE with prior and posterior log-odds. In other words, one may consider the WoE as how much to update the prior log odd considering the evidence $e$. To simplify the probability calculations, in this paper we assume that $h'=\bar{h}$ is the complementary explanation, yet the framework applies to any alternative hypothesis $h'$.

\subsection{Weight of Evidence for Explanation Qualification}
\label{sec:woe_framework}
 
We apply the WoE framework to $N$ executions of a FIM on a fixed instance $x$. Let $\alpha^{(j)} = (\alpha_1^{(j)},\dots,\alpha_d^{(j)}) \in \mathbb{R}^d$ denote the attribution vector produced at run $j = 1,\dots,N$. The set of runs provides empirical evidence about the explanatory behavior of the FIM for the given instance.

Inspired by rule-based formulations \cite{rudin2023globally,geng2022computing}, from each attribution vector one derives a binary hypothesis about which features are relevant. Specifically, we consider the set of relevant and non-relevant features:
\begin{equation*}
    \mathcal{R}_x = \{ i : \text{feature } i \text{ is relevant} \},
    \qquad
    \mathcal{N}_x \subseteq \{1,\dots,d\} \setminus \mathcal{R}_x.
\end{equation*}
Invoking the principle of sparsity, we focus only on the relevant set and write the hypothesis as:
\begin{equation}
\label{eq:hypothesis}
    h = \bigwedge_{i \in \mathcal{R}_x} (i \text{ is relevant}).
\end{equation}
The precise procedure mapping a continuous attribution vector $\alpha$ to a set $\mathcal{R}_x$ (and hence to a hypothesis) is detailed afterwards in Section \ref{sec:feature_selection}. Let us demonstrate a first result that will then allow us to obtain the general condition for this to happen. Importantly, the reference hypothesis (that is, the one we want to check whether the FIM supports it or not) $h^*$ can originate from two sources. First, it can be derived from the evidence itself by applying the selection procedure to the empirical mean attribution $\bar{\alpha} = \frac{1}{N}\sum_{j=1}^N \alpha^{(j)}$; in this case WoE captures the \emph{explainer stability}\footnote{\label{note1}We have chosen to refer to these properties by convention to explain what aspects we are measuring. In the case of stability, which has been formalized in various ways, in this paper it coincides with the variability of the explainer.}. Second, it can reflect external information, such as ground-truth relevant features or domain knowledge; in this case WoE measures the \emph{alignment}\footnotemark[\value{footnote}] between the FIM and the reference explanation, while accounting for the aforementioned FIM's intrinsic variability.

\subsection{Quantifying the Weight of Evidence for a Stochastic FIM}
\label{sec:quantifying_woe}
 
Within this framework, WoE measures how strongly the empirical evidence supports the hypothesis $h$ over an alternative, thus providing a contrastive, quantitative criterion for explanation qualification.

Using equation \eqref{eq:woe_identity}, quantifying the WoE reduces to measuring the prior and posterior log-odds. When explaining a black-box model with a stochastic FIM (e.g. LIME \cite{LIME}, SHAP \cite{lundberg2017unified}), the explanation can be viewed as a random variable $\alpha$ induced by the randomness of the FIM. 

\paragraph{Prior probabilities.}
In the absence of prior information, we assume a uniform distribution over all $2^d$ possible binary assignments of relevance to the $d$ features. The hypothesis $h$ in equation~\eqref{eq:hypothesis} is \emph{partial}: it asserts that the features in $\mathcal{R}_x$ are relevant, but makes no claim about the remaining $d - |\mathcal{R}_x|$ features. Consequently, the number of binary vectors compatible with $h$ is $2^{d - |\mathcal{R}_x|}$, yielding:
\[
    p(h) = \frac{2^{d-|\mathcal{R}_x|}}{2^d} = \frac{1}{2^{|\mathcal{R}_x|}}.
\]
Intuitively, a hypothesis designating fewer features as relevant is assigned higher prior probability, reflecting greater conservativeness. Importantly this is not the only admissible prior. The strategy we define can accommodate any prior derived from different assumptions or knowledge.

\paragraph{Posterior probabilities.}
To estimate $p(h \mid e)$, we leverage the stochasticity of the FIM. Running it $N$ times on the same instance $x$ yields attribution vectors $\alpha^{(1)},\dots,\alpha^{(N)}$, from each of which a hypothesis $h^{(j)}$ is derived via the selection procedure. The posterior probability is estimated empirically as the proportion of runs that reproduce $h$:
\begin{equation} \label{eq:posterior}
    \hat{p}(h \mid e) = \frac{1}{N} \sum_{j=1}^N \mathbf{1}\{ h^{(j)} = h \}.
\end{equation}
This Monte Carlo estimate allows us to compute the posterior log-odds and, consequently, the WoE.
 
\subsection{Feature Selection and Hypothesis Definition}
\label{sec:feature_selection}
In this section, we explain some possibilities for generating hypotheses from an FIM. These choices are common strategies, and other methods may be considered depending on the context and type of analysis. For simplicity, let $\alpha \in \mathbb{R}^d$ denote a run of an FIM, such that $\alpha_i\geq0$ (e.g. by the taking the absolute value). Let $\pi$ the permutation map which sorts the indices of the attribute vector such that $\alpha_{\pi(i)}\geq \alpha_{\pi(j)}$ if $i<j$.

 
\paragraph{1. Top-$k$ selection (TK).}
\label{sec:top_k_selection}
Fix $k \in \{1,\dots,d\}$ and set:
\[
    \mathcal{R}_x = \{\pi(1),\dots,\pi(k)\},
\]
i.e., we retain the first $k$ features according to the ordering induced by the attribution scores.
 
\paragraph{2. Adaptive selection (AS).} \label{sec:adapative_selection}
The previous strategy require knowledge of the FIM scale or make assumptions about how many variables to consider. We propose an adaptive strategy based on imposing a \textit{sufficiency constraint}: the selected features explain at least a fraction $\tau$ of the total attribution mass for the instance. Assuming, for simplicity, that the attribution scores are normalized so that $\sum_{i=1}^d \alpha_i = 1$, we define
\begin{equation} \label{eq:cumulative_importance}
    k^*(\alpha) = \min\!\left\{ k \in \{1,\dots,d\} : \sum_{i=1}^k \alpha_{\pi(i)} \geq \tau \right\},
\end{equation}
and set $\mathcal{R}_x = \{\pi(1),\dots,\pi(k^*(\alpha))\}$. Throughout the paper, $\tau$ is assumed to be in the interval $(0,1)$ and the attribution vector is normalized to $1$ to simplify the proofs. Also to simplify the notation, we relabel the indices so that the attribution vector is already sorted in descending order, i.e., $\alpha_1\geq\ldots\geq\alpha_d$.
 
In both scenarios, the hypothesis derived from the j-th run coincides with the set of relevant features $\mathcal{R}_x$.
 
\section{Connecting WoE with Attribution Variability}
\label{sec:variability}
The aim of this section is to mathematically characterize what the WoE is measuring when the reference hypothesis $h^*$ is generated from the mean attribution vector over $N$ runs. The main result is that as attribution variability goes to $0$, WoE diverges to $+\infty$. Therefore, in this setting, the WoE provides a direct quantification of the variability of the FIM.
 
Recall from equation~\ref{eq:posterior} that $\hat{p}(h^* \mid e)$ measures how frequently the FIM reproduces $h^*$ across runs. Intuitively, if the attribution vectors concentrate around their means, then both the selected feature set and its cardinality remain stable, so nearly every run reproduces $h^*$ and $\hat{p}(h^* \mid e) \to 1$.
 
We formalize this in two steps: first we show that $k^*$ is stable (Section~\ref{sec:stability_k})---which is needed in the case of the AS strategy---, then that the selected feature set is stable (Section~\ref{sec:stability_set}). We begin by recalling a standard probability inequality.
 
\begin{lemma} \label{lemma:events}
For any events $A_1,\dots,A_n$,
\[
    p\!\left(\bigcap_{i=1}^n A_i\right) \geq 1 - \sum_{i=1}^n p(A_i^c).
\]
\end{lemma}
 
\subsection{Stability of the Number of Selected Features}
\label{sec:stability_k}
 
For TK the cardinality $|\mathcal{R}_x| = k$ is fixed by construction, so there is nothing to prove. For AS, the value $k^*(\alpha^{(j)})$ may vary across runs; the following result shows it concentrates on $k^\dagger := k^*(\bar{\alpha})$ when the attribution variances are small.
 
\begin{proposition}[Stability of $k^*$ under low variability]
\label{prop:kdagger}
Let $\alpha = (\alpha_1,\dots,\alpha_d)$ be a mutually independent, stochastic FIM with expected values $\mu_i = \mathbb{E}[\alpha_i]$, and variances $\sigma_i^2 = \mathrm{Var}(\alpha_i)$. Further assume that $\sum_{i=1}^d \alpha_i = 1$ and that the expected values are pairwise distinct and ordered so that $\mu_1 > \cdots > \mu_d$. Let $k^\dagger$ be the solution to~\eqref{eq:cumulative_importance} computed at $\mu = (\mu_1,\dots,\mu_d)$, and suppose there exists a margin $\delta > 0$ such that:
\begin{equation} \label{eq:margin_assumption}
        \sum_{i=1}^{k^\dagger} \mu_i \geq \tau + \delta
    \qquad\text{and}\qquad
    \sum_{i=1}^{k^\dagger-1} \mu_i \leq \tau - \delta.
\end{equation}
Then for every $\varepsilon \in (0,1)$:
\[
    \sum_{i=1}^{k^\dagger} \sigma_i^2 \leq \varepsilon\,\frac{\delta^2}{2}
    \implies
    p\!\left(k^*(\alpha) = k^\dagger\right) \geq 1 - \varepsilon.
\]
\end{proposition}
 
\begin{proof}
Without loss of generality assume features are already ordered by decreasing mean, so $\pi_\mu = \mathrm{id}$. Define the partial sums $S_k = \sum_{i=1}^k \alpha_i$. By mutual independence of the $\alpha_i$ and Chebyshev's inequality:
\[
    p\!\left(|S_k - \sum_{i=1}^k \mu_i| \geq \delta\right) \leq \frac{\sum_{i=1}^k \sigma_i^2}{\delta^2} \leq \frac{\varepsilon}{2},
\]
since $\sum_{i=1}^{k^\dagger} \sigma_i^2 \leq \varepsilon \frac{\delta^2}{2}$. Similarly, $\sum_{i=1}^{k^\dagger - 1}\sigma_i^2 \leq \sum_{i=1}^{k^\dagger}\sigma_i^2 \leq \varepsilon \frac{\delta^2}{2}$, so both events
\[
    A = \left\{|S_{k^\dagger} - \sum_{i=1}^{k^\dagger}\mu_i| < \delta\right\}
    \quad\text{and}\quad
    B = \left\{|S_{k^\dagger-1} - \sum_{i=1}^{k^\dagger-1}\mu_i| < \delta\right\}
\]
hold each with probability at least $1 - \varepsilon/2$. On the event $A \cap B$, the margin assumption (eq. \ref{eq:margin_assumption}) gives:
\[
    S_{k^\dagger} > \tau + \delta - \delta = \tau
    \qquad\text{and}\qquad
    S_{k^\dagger-1} < \tau - \delta + \delta = \tau,
\]
which implies $k^*(\alpha) = k^\dagger$. Applying Lemma~\ref{lemma:events}:
\[
    p(k^*(\alpha) = k^\dagger) \geq p(A \cap B) \geq 1 - \varepsilon.  \qed
\]
\end{proof}
 
Proposition~\ref{prop:kdagger} guarantees that, under low variability, repeated runs yield a stable value of $k^*(\alpha)$, so the cardinality $|\mathcal{R}_x^{(j)}|$ should remain constant across runs.
 
\subsection{Stability of the Selected Feature Set}
\label{sec:stability_set}
 
We now show that low variability also implies stability of the identity of the selected features, i.e.\ the same features are ranked highest across runs with high probability. This reduces to ensuring that every relevant feature consistently outranks every non-relevant one.
 
\begin{lemma} \label{lemma:two_ordering}
Let $\alpha_i, \alpha_j$ be independent with $i<j$ and expected values $\mu_i > \mu_j$ and variances $\sigma_i^2, \sigma_j^2$. If
\[
    \sigma_i^2 + \sigma_j^2 \leq \varepsilon\,(\mu_i - \mu_j)^2,
\]
then $p(\alpha_i > \alpha_j) \geq 1 - \varepsilon$.
\end{lemma}
 
\begin{proof}
Let $D = \alpha_i - \alpha_j$, with expected value $\Delta = \mu_i - \mu_j > 0$ and variance $\sigma_i^2 + \sigma_j^2$. By Chebyshev's inequality:
\[
    p(|D - \Delta| \geq \Delta) \leq \frac{\sigma_i^2 + \sigma_j^2}{\Delta^2} \leq \varepsilon.
\]
Observe that $|D - \Delta| < \Delta$ implies $D > 0$, which means that $p(\alpha_i > \alpha_j) = p(D>0)\geq p(|D-\Delta|<\Delta) \geq 1 - \varepsilon$. \qed
\end{proof}
 
\begin{corollary} \label{cor:order}
Let $\alpha_1,\dots,\alpha_d$ be pairwise independent with means $\mu_1 > \mu_2 > \cdots > \mu_d$. If for every $i \in \{1,\dots,k^\dagger\}$ and $j \in \{k^\dagger+1,\dots,d\}$:
\[
    \sigma_i^2 + \sigma_j^2 \leq \varepsilon_{i,j}\,(\mu_i - \mu_j)^2,
\]
then:
\[
    p\!\left(\mathcal{R}_x^{(j)} = \{1,\dots,k^\dagger\}\right) \geq 1 - \sum_{i=1}^{k^\dagger}\sum_{j=k^\dagger+1}^{d} \varepsilon_{i,j}.
\]
\end{corollary}
 
The proof follows immediately from Lemma~\ref{lemma:events} and Lemma~\ref{lemma:two_ordering} applied to each pair $(i,j)$.
 
\subsection{Divergence of WoE under Low Variability}
To unify notation, in the TK scenario we refer to $k$ as $k^\dagger$. Let $A$ be the event that $k^{(j)} = k^\dagger$ (assured in the case of TK), and $B$ be the event that the correct $k^\dagger$ features are ranked highest. Combining the results of Sections~\ref{sec:stability_k}--\ref{sec:stability_set}:
\[
    p(A) \geq 1 - \varepsilon_k, \qquad p(B) \geq 1 - \varepsilon_o,
\]
where $\varepsilon_k$ (which is $0$ in the case of TK) and $\varepsilon_o$ can be made arbitrarily small by reducing the attribution variances through increased sampling. Since $A \cap B$ implies $h^{(j)} = h^*$, Lemma~\ref{lemma:events} gives:
\[
    p_0 := p(h^{(j)} = h^*) \geq 1 - \varepsilon_k - \varepsilon_o.
\]
 
Now let $X$ be the number of runs (out of $N$) for which $h^{(j)} = h^*$. By independence of runs, $X \sim \mathrm{Bin}(N, p)$ for some $p \in [0,1]$. Our goal is  to estimate the empirical posterior and, in particular, to characterize the probability of observing large values of the posterior probability. As we will show, by sufficiently controlling the variance, we can guarantee with high probability that the posterior exceeds an arbitrary threshold. It is immediate to observe that the empirical posterior in eq. \ref{eq:posterior} satisfies $\hat{p}(h^* \mid e) = X/N \geq \lambda$ if and only if $X \geq \lceil \lambda N \rceil$.

\begin{lemma} \label{lemma:tail_decreasing}
Let $X \sim \mathrm{Bin}(N, p)$ and $k \in \{1,\dots,N\}$. The function $T(p) := P(X \geq k)$ is strictly increasing in $p \in (0,1)$.
\end{lemma} 
\begin{proof}
By definition, $T(p):= P(X \geq k) = \sum_{i=k}^N \binom{N}{i}p^i(1-p)^{N-i}$. Differentiating:
\begin{align*}
    T'(p) &= \sum_{i=k}^{N-1} \binom{N}{i}\left[ip^{i-1}(1-p)^{N-i} - (N-i)p^i(1-p)^{N-i-1}\right] +Np^{N-1} \\
    &= \frac{1}{p(1-p)}\sum_{i=k}^N \binom{N}{i}(i - Np)\,p^i(1-p)^{N-i}= \frac{1}{p(1-p)}\,\mathbb{E}\!\left[(X - \mathbb{E}[X])\,\mathbf{1}_{X \geq k}\right].
\end{align*}
Since $\mathbb{E}[(X - \mathbb{E}[X])] = 0$, we have:
\[
\mathbb{E}\!\left[(X-Np)\,\mathbf{1}_{X \geq k}\right] = -\,\mathbb{E}\!\left[(X-Np)\,\mathbf{1}_{X < k}\right].
\]
If $k \leq Np$, then for all $i < k$ we have $(i - Np) < 0$, hence $\mathbb{E}[(X-Np)\,\mathbf{1}_{X < k}] < 0$, and therefore $\mathbb{E}[(X-Np)\,\mathbf{1}_{X \geq k}] > 0$. If instead $k > Np$, then for all $i \geq k$ we have $(i - Np) > 0$, hence $\mathbb{E}[(X-Np)\,\mathbf{1}_{X \geq k}] > 0$. In both cases:
\[
\frac{\partial T}{\partial p} = \frac{1}{p(1-p)}\,\mathbb{E}\!\left[(X-Np)\,\mathbf{1}_{X \geq k}\right] > 0. 
\]
\qed
\end{proof}
Applying Lemma~\ref{lemma:tail_decreasing} with $p \geq 1 - \varepsilon_k - \varepsilon_o =: p_l$:
\[
    p(X \geq \lceil\lambda N\rceil) \geq \sum_{i=\lceil\lambda N\rceil}^{N}\binom{N}{i}p_l^i(1-p_l)^{N-i} =: p_\lambda.
\]
Now, we observe that $1\geq p_{\lambda}\geq p_l^N$ where the second inequality is obtained by setting $i = N$ and noting that the terms of the sum are all positive. Hence, as $p_l \to 1$, $p_\lambda \to 1$ for any fixed $\lambda \in (0,1)$.
 
\paragraph{Implication for WoE.}
Under sufficiently low attribution variability, $\hat{p}(h^* \mid e) \geq \lambda$ with probability $p_\lambda$. Now, considering the complementary hypothesis $\bar{h}^*$ as the alternative one, $\hat{p}(\bar{h}^* \mid e) \leq 1 - \lambda$. Since the prior odds term is bounded, it follows that, as $\lambda\to 1$:
\[
woe(h^*:e)\to \infty.
\]
This establishes the following.
 
\begin{theorem} \label{theo:infinity}
Let $\alpha$ be a stochastic mutually independent FIM with component variances $\sigma_i^2$. Fix an instance $x$ and let $h^*$ be the hypothesis generated from the mean attribution $\bar{\alpha}$ via AS or TK (with the only difference of $\epsilon_k=0)$ and let $h'$ be the complementary hypothesis. Then, as $\max_i \sigma_i^2 \to 0$,
\[
woe(h^* : e) \to +\infty
\]
with probability approaching $1$.
\end{theorem}

\section{Experimental Setup}
\label{sec:experiments}
This section is structured according to the origin of the reference hypothesis.  \\ (1) We test LIME and SHAP against a hypothesis $h^*$ derived from domain knowledge on the Titanic dataset \cite{cukierski2012titanic}, quantifying its alignment with FIMs. While a FIM produces a local explanation, the domain hypothesis is global. This reveals that WoE is affected by the global–local discrepancy (Section~\ref{sec:domain_knowledge}). (2) We consider a ground truth (GT) setting by generating a synthetic dataset and training a RF, thus having access to both data- and model-level relevance. This allows us to assess alignment (w.r.t. the data) and faithfulness (w.r.t. the model) via WoE (Section~\ref{sec:GT_knowledge}). 3) We test a third scenario where no external references are available. In this case, $h^*$ is defined as the empirical mean attribution across runs, so WoE captures internal stability (Section \ref{sec:hyp_testing}). The section concludes with an empirical validation of Theorem \ref{theo:infinity} and an analysis of the impact of the number of explainer runs (Section \ref{sec:empirical_validation}). 

Notably, to generate $h_j$ (Section \ref{sec:top_k_selection}), we adopt top-k selection  in the first two settings, as the exact set of relevant features is known; while in the third, we rely on the adaptive strategy, since no prior information on feature relevance is available.
Due to space constraints, details on datasets, preprocessing, and model hyperparameters will be left for the final version.

\subsection{Hypothesis Testing With Domain Knowledge}
\label{sec:domain_knowledge}
In this section, we compute the WoE for 50 instances LIME and SHAP on the Titanic dataset using a RF model.
Historical analyses of this setting identify \textit{sex} (due to the “women first” policy) and \textit{pclass} (a proxy for socio-economic status) as the most influential factors for survival, as also confirmed by prior empirical studies \cite{daml24}. Taking this as our reference hypothesis, our results show that LIME produces $27$ instances with $woe = +\infty$ (i.e., $p(h \mid e)=1$, where the reference hypothesis is always generated) and $23$ cases with finite $woe > 0$. In contrast, SHAP yields deterministic explanations, with $23$ instances having $woe = +\infty$ and $27$ having $woe = -\infty$ (i.e., $p(h \mid e)=0$), indicating the selection of features beyond $sex$ and $pclass$.

\begin{figure}[!h]
    \centering
    \begin{subfigure}{0.48\columnwidth}
        \centering
        \includegraphics[width=\linewidth]{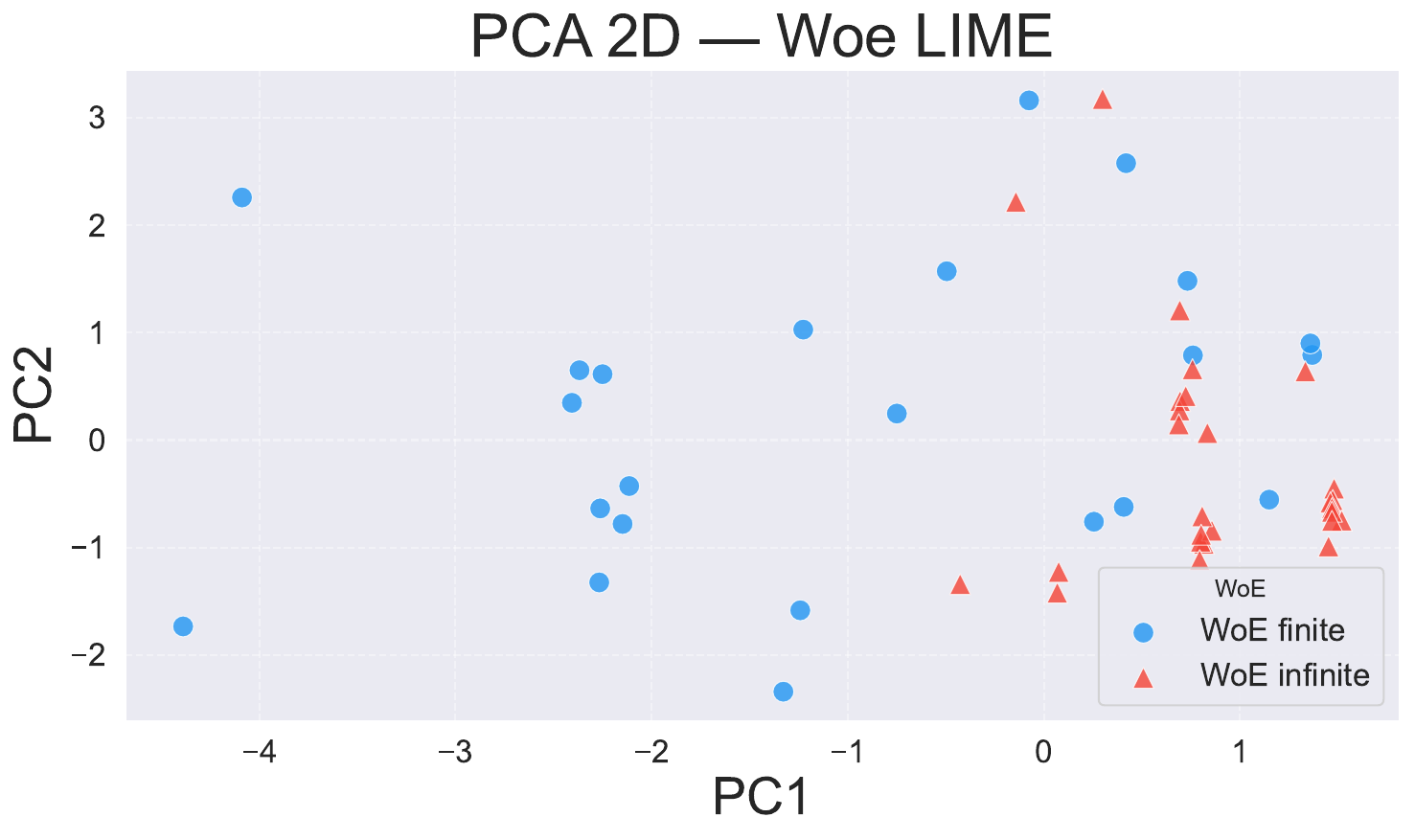}
    \end{subfigure}
    \hfill
    \begin{subfigure}{0.48\columnwidth}
        \centering
        \includegraphics[width=\linewidth]{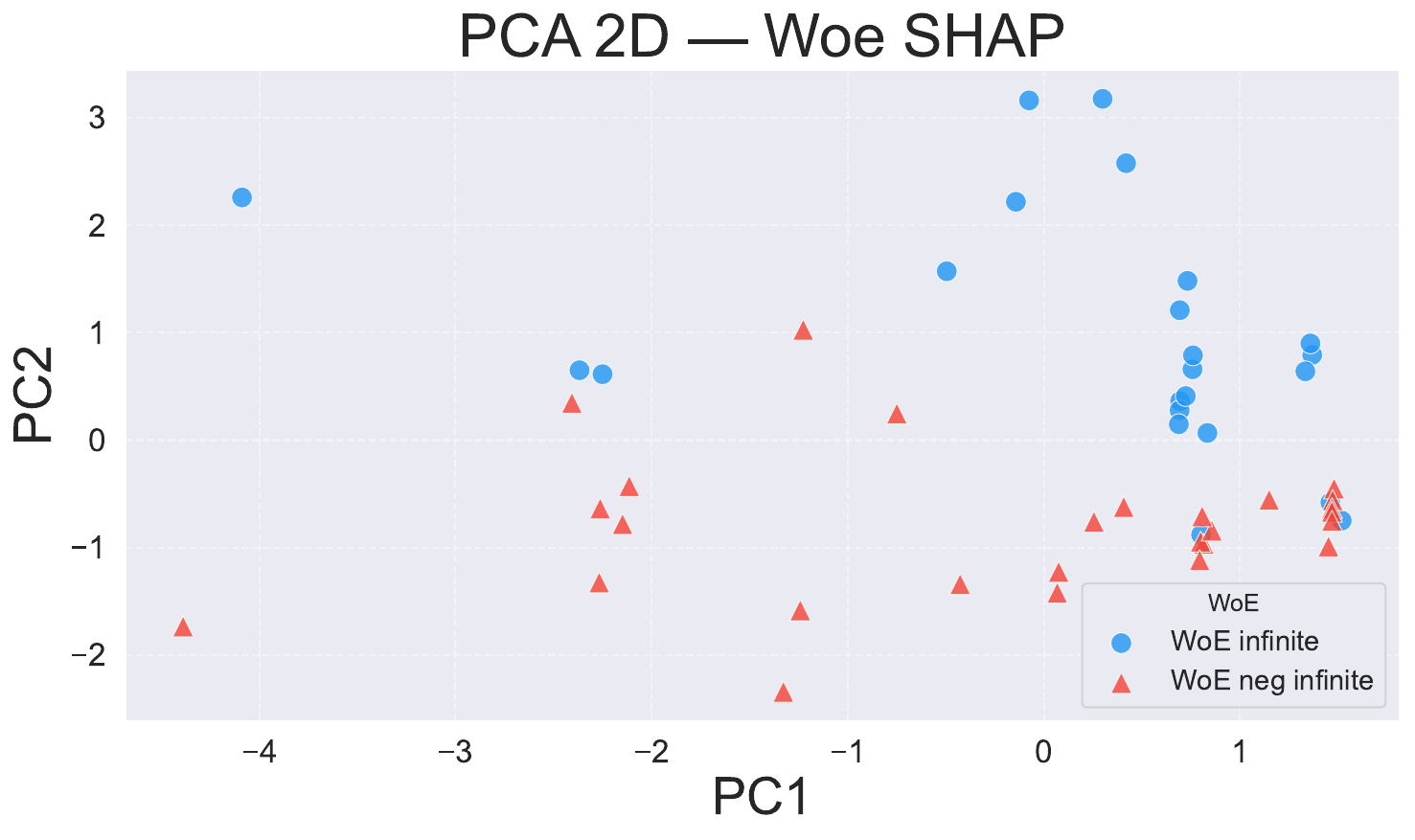}
    \end{subfigure}
    \caption{PCA Analysis for LIME and SHAP grouped according to WoE values.}
    \label{fig:pcas}
\end{figure}

A PCA projection with $54\%$ of variance explained (\autoref{fig:pcas}) further highlights geometric differences: for LIME, instances with finite and divergent WoE occupy distinct regions, while for SHAP the separation between $+\infty$ and $-\infty$ cases is even more pronounced. This suggests that in specific regions of the feature space, alternative variables (e.g., age) influence the explanations, contrary to the assumption of a single global explanation.

\textbf{Insight}: through WoE, we can test alignment between a local explanation and an assumed global explanation, thereby identifying regions where the local behavior differs.

\subsection{Hypothesis Testing with Synthetic Ground Truth Knowledge}
\label{sec:GT_knowledge}

In this section, we employ the WoE framework to measure alignment with Ground Truth (GT) knowledge. To this end, we construct a synthetic classification dataset, using the \texttt{make\_classification} function from \texttt{scikit-learn}---which generates a random n-class classification problem---, with $n_{\mathrm{inf}}=3$ informative features, whose indices $\{0,1,2\}$ constitute the known ground truth $h^*_{GT}$ by construction, and $n_{\mathrm{red}}=7$ redundant features generated as:
\[
     X_{\mathrm{red}}
     =
     (1-\lambda)\,X_{\mathrm{inf}}\,W + \lambda\,\mathcal{E},
     \quad
     W \in \mathbb{R}^{3 \times 7},\; \mathcal{E}\in \mathbb{R}^{n\times 7},\,
     \mathcal{E} \sim \mathcal{N}(0, I_7),
\]
where $\lambda \in [0,1]$ interpolates between a perfect linear combination of the informative features ($\lambda=0$) and pure noise ($\lambda=1$) and $n$ represents the amount of samples. We evaluate LIME and SHAP on a RF with $N=50$ runs on $50$ instances using top-$k$ selection ($k=3$), and report: $WoE_{GT}$, measuring alignment with $h^*_{GT}$; and $WoE_{Model}$, measuring alignment with the model's feature ranking $h^*_{Model}$ (i.e., faithfulness). The results for LIME are summarized in \autoref{fig:synthetic_gt}.

\begin{figure}[!h]
    \centering
    \includegraphics[width=1\linewidth]{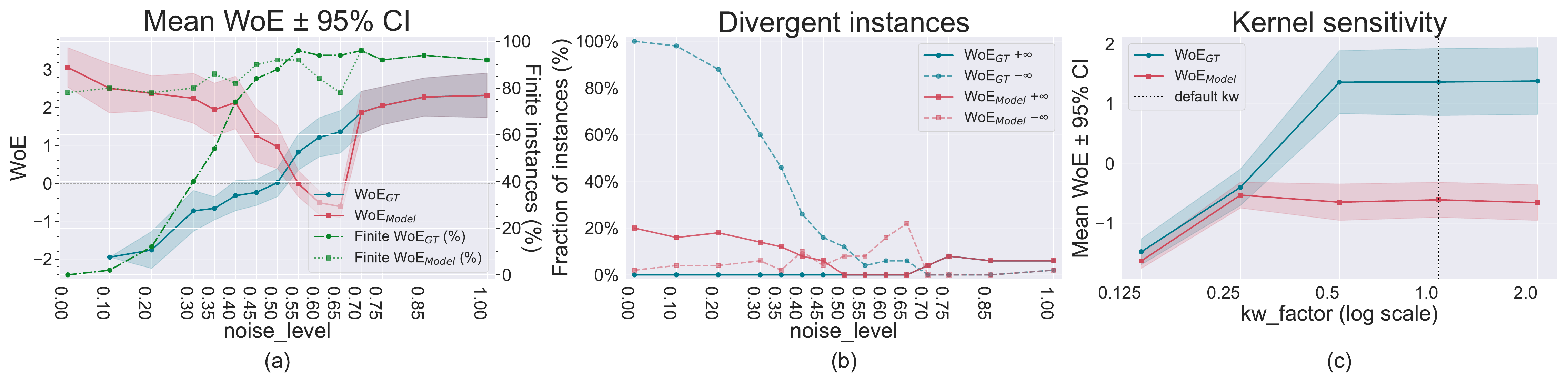}
    \caption{LIME experiment with known GT. (a) Mean $\mathrm{WoE}_{GT}$ and $\mathrm{WoE}_{Model}$ $\pm$ 95\% CI, with fraction of finite WoE values on the right vertical axis. (b) Fraction of divergent instances as a function of $\lambda$. (c) Sensitivity of $\mathrm{WoE}_{GT}$ and 
$\mathrm{WoE}_{Model}$ to LIME's kernel-width factor ($kw\_factor$) at 
$\lambda=0.65$, with the vertical line marking the default value.}
     \label{fig:synthetic_gt}
\end{figure}

This allows us to measure the alignment FIM-Model and FIM-GT. Two opposite trends emerge: at low noise levels, LIME aligns much more with the model than with the GT, showing that the model relies more on redundant features than on the truly relevant ones. As noise increases ($\lambda \geq 0.70$), $WoE_{GT}$ and $WoE_{Model}$ converge, meaning that redundant features become too noisy and the model is forced to rely on the truly relevant features. 

Counterintuitively, when $\lambda \in (0.50, 0.70)$, the FIM appears more aligned with the GT than with the model. To better understand this phenomenon, we analyzed the case $\lambda = 0.65$. As shown in \autoref{fig:synthetic_gt}(c), reducing the kernel width by a multiplicative factor ($kw\_factor$) causes $WoE_{Model}$ to progressively align with $WoE_{GT}$. A possible explanation for this behavior is the high curvature of the local neighborhood around the sample to be explained, which LIME ignores when its sampling radius is too high, thus leading to unfaithful explanations. However, our WoE framework manages to diagnose this.

SHAP plots are omitted because they are not very informative: since its values are almost deterministic, they lead to either complete alignment/disalignment w.r.t. the model and the GT, resulting in few finite instances to plot. However, the same trend is observed: as noise increases, the model converges to the GT behavior and the WoE values capture this pattern.

\textbf{Insight:} WoE can localize the disalignment along the Data$\,\to\,$Model$\,\to\,$FIM pipeline, thereby assessing the overall reliability of the FIM.

\subsection{Hypothesis Testing Without Prior Knowledge} \label{sec:hyp_testing}
The experiments in this section focus on comparing WoE across four datasets: Diabetes~\cite{diabetes_dataset}, Heart Disease~\cite{heart_disease}, Mobile~\cite{Mobile_price}, and Churn~\cite{churn_dataset}. Here we opted for the AS strategy with $\tau=0.8$. \autoref{fig:woe_distributions} illustrates the distributions obtained for LIME and SHAP. For the latter, Diabetes and Churn are not presented because SHAP always produced the same hypotheses, yielding $woe = +\infty$ (\autoref{tab:woe_divergence}). Overall, there are no significant differences between the two distributions, demonstrating that the two methods are comparable. However, when analyzing \autoref{tab:woe_divergence} of divergent WoE, we observe that SHAP tends to produce the same explanation more consistently, a finding consistent with prior works \cite{stability_perturbation,nayebi2023empirical}. Interestingly, WoE generally presents a positive score, demonstrating that LIME and SHAP produce internally consistent explanations. However, this score decreases for more complex datasets, as is the case with Mobile, and where LIME even shows 0 positive divergent values and 5 negative divergent values. This is likely due to the greater interactions between the variables, which make it more difficult to obtain a stable ranking and thus generate the same hypothesis.

\textbf{Insight:} when evaluating WoE on mean feature importance, we can capture the internal consistency of the FIM---i.e. its variability.

\begin{figure}[!h]
    \centering
    \begin{subfigure}{0.49\columnwidth}
        \centering
        \includegraphics[width=\linewidth]{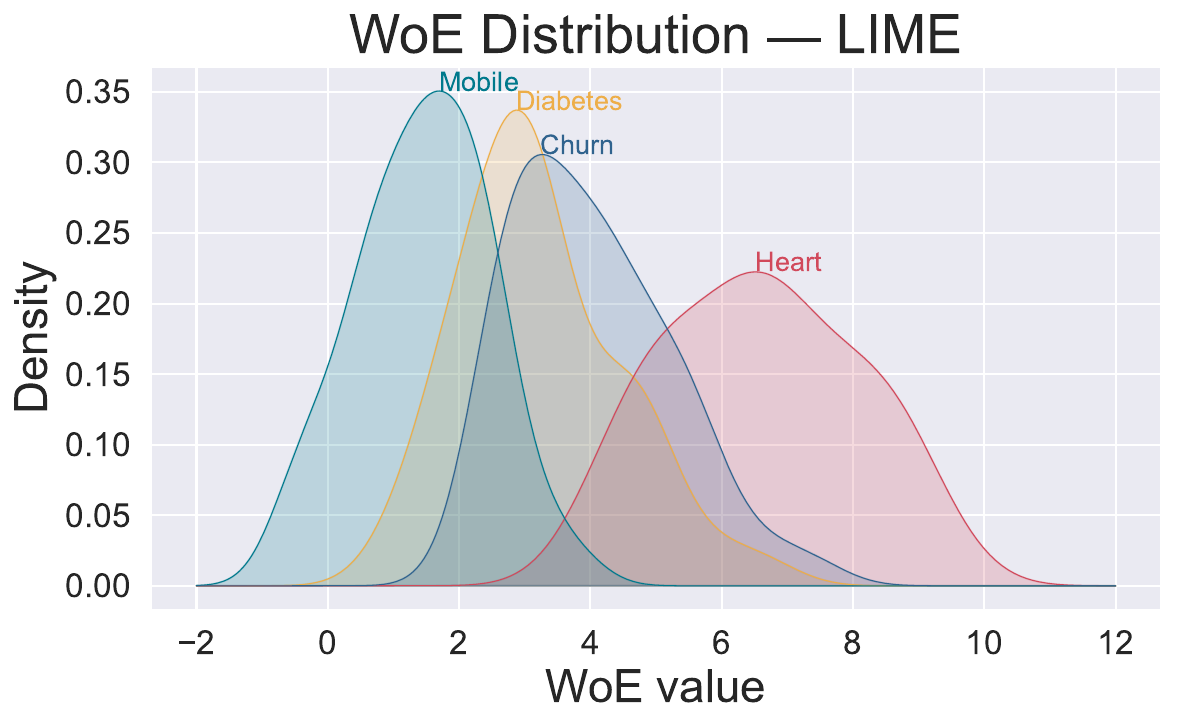}
    \end{subfigure}
    \hfill
    \begin{subfigure}{0.49\columnwidth}
        \centering
        \includegraphics[width=\linewidth]{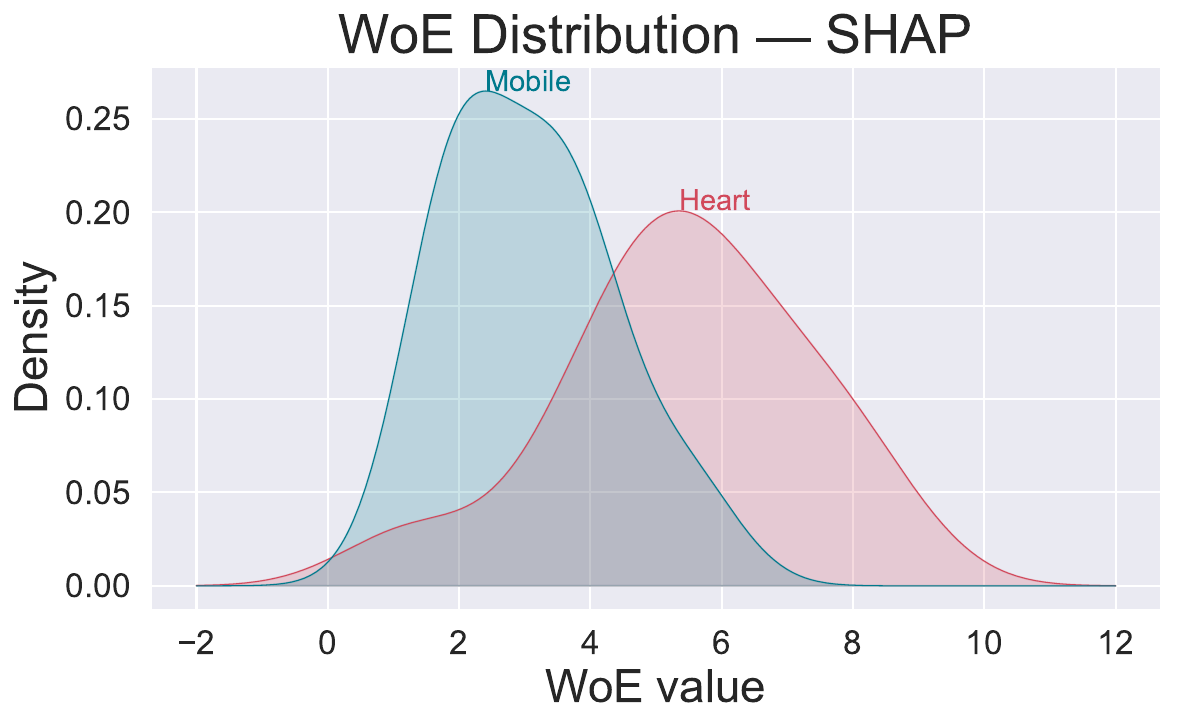}
    \end{subfigure}
    \caption{Gaussian KDE of the LIME and SHAP WoE values across datasets.}
    \label{fig:woe_distributions}
\end{figure}

\begin{table}[t]
\caption{Number of divergent WoE cases out of $50$ instances for each method and dataset. Each entry is reported as (+$\infty$ / -$\infty$).\\}
\centering
\begin{tabular}{lcccc}
\hline
Method & Diabetes & Heart & Mobile & Churn \\
\hline
LIME  & 5 / 0  & 11 / 0 & 0 / 5  & 5 / 0 \\
SHAP  & $\boldsymbol{50 / 0}$ & $\boldsymbol{25 / 0}$ & $\boldsymbol{41 / 0}$ & $\boldsymbol{50 / 0}$ \\
\hline
\end{tabular}
\label{tab:woe_divergence}
\end{table}

\subsection{Empirical Validation of Theorem \ref{theo:infinity}}
\label{sec:empirical_validation}
We empirically validate that, when using the mean attribution as the reference hypothesis, lower variance in AMs leads to higher WoE. We consider two settings: (i) real data, using LIME on the Diabetes dataset while varying the number of background samples for 4 models: Random Forest (RF), Multilayer perceptron (NN), Decision Tree (DT) and Support Vector Machine (SVM);  and (ii) a synthetic FIM with controlled variance. In the first setting (\autoref{fig:LIME_var}), increasing $num\_samples$, which ranges in \{100,\allowbreak 200,\allowbreak 300,\allowbreak 500,\allowbreak 1000,\allowbreak 2000,\allowbreak 3000,\allowbreak 5000\}, reduces variance and results in higher WoE values. This is also reflected in the proportion of divergent cases: positively divergent WoE ($+\infty$) slightly increases, while negatively divergent cases ($-\infty$) decrease significantly. In the synthetic setting (\autoref{fig:syntehtic_var_boxplot}), we isolate the effect of variability by directly generating synthetic attribution vectors, without relying on any dataset or model. We define a reference importance vector with $10$ features evenly spaced in $[0,1]$, and model the FIM as a stochastic generator centered around these values. For each configuration, we sample $50$ attribution vectors for each of $50$ simulated instances. Variability is controlled by adding Gaussian noise with standard deviation $\sigma = (1/l)\cdot 0.1$, where $l \in {0.5, 1, 1.5, 2}$.
The results consistently show that lower variance leads to higher WoE, confirming the theoretical result.

\begin{figure}[!h]
\centering
\begin{minipage}{0.48\columnwidth}
    \includegraphics[width=\linewidth]{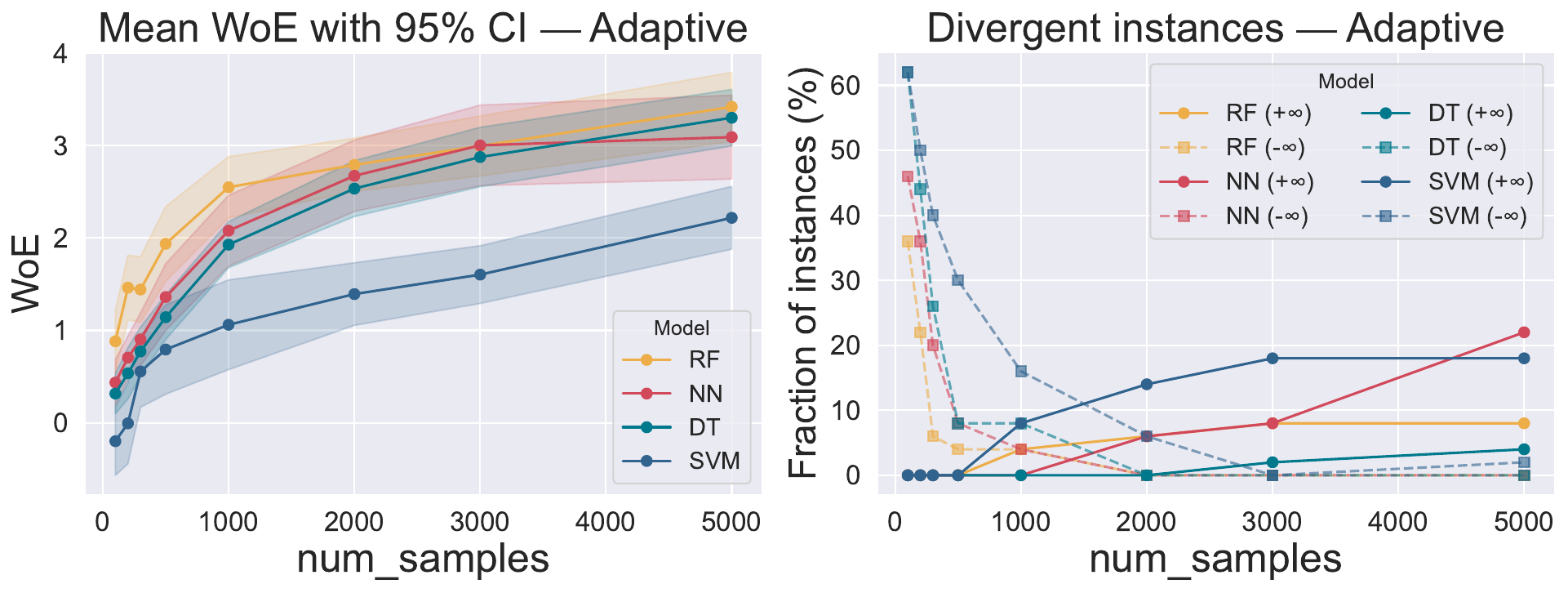}
    \caption{Impact of the variability on the WoE for LIME across 4 models on the Diabetes dataset.}
    \label{fig:LIME_var}
\end{minipage}\hfill
\begin{minipage}{0.48\columnwidth}
    \includegraphics[width=\linewidth]{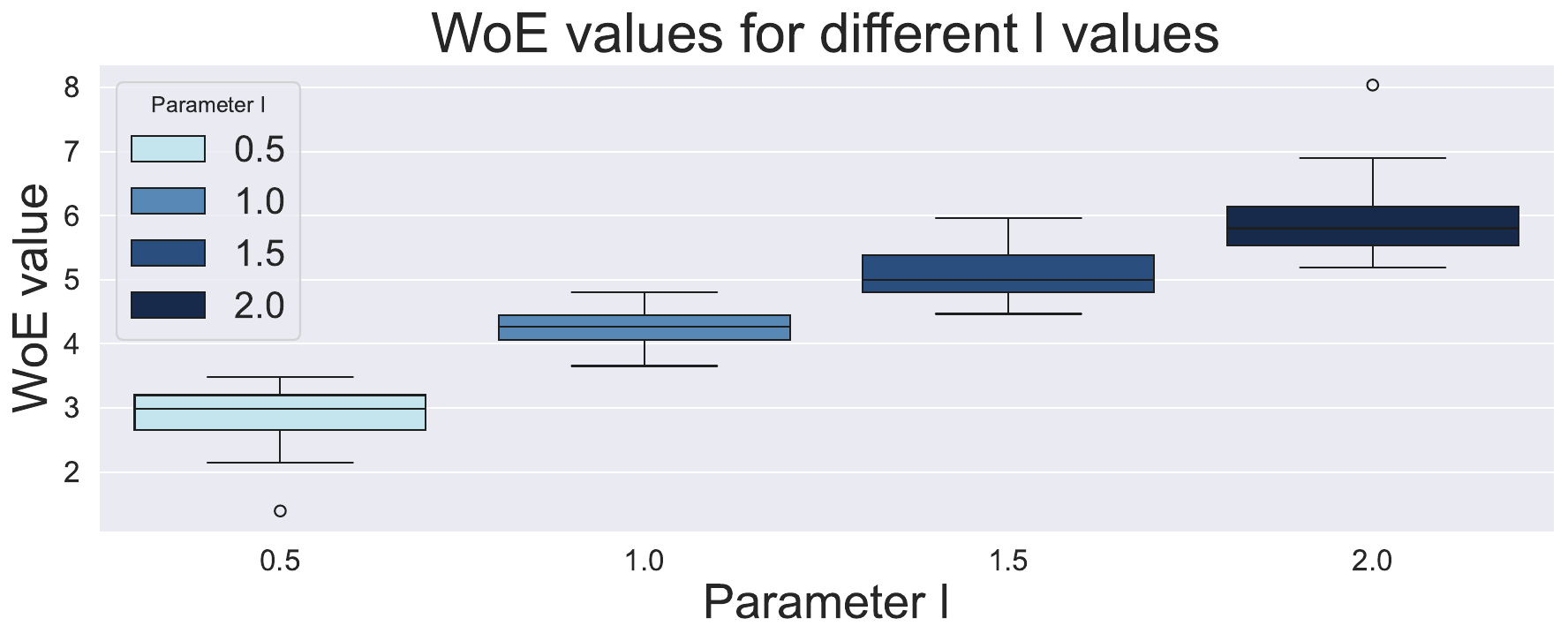}
    \caption{Analysis of the impact of variability in the WoE for a controlled synthetic setting.}
    \label{fig:syntehtic_var_boxplot}
\end{minipage}
\end{figure}

\subsubsection{Impact of the Number of Runs}
Our estimate of the posterior distribution is based on the proportion of runs that support the reference hypothesis. In this section, we briefly analyze the impact of the number of runs $N$ on WoE estimates. We report the RF results for Diabetes using LIME. We run LIME $N$ times on $50$ instances of the test set with $N \in \{5, 10, 20, 30, 40, 50, 60, 70, 80, 90, 100\}$. We then analyzed the distributions in \autoref{fig:woe_runs} of the observed WoE values, noting that there is high consistency across the runs. Furthermore, we measured the dispersion of the WoE vectors (see \autoref{fig:woe_cv}) for each instance using the coefficient of variation (CV), which is defined as $\sigma/|\mu|$. The distribution of CV values is concentrated around values close to $0$ (with 80\% of values between $0$ and $0.3$), indicating consistency in the weight of evidence values across the runs. 

\begin{figure}[!h]
\centering
\begin{minipage}{0.48\columnwidth}
    \includegraphics[width=\linewidth]{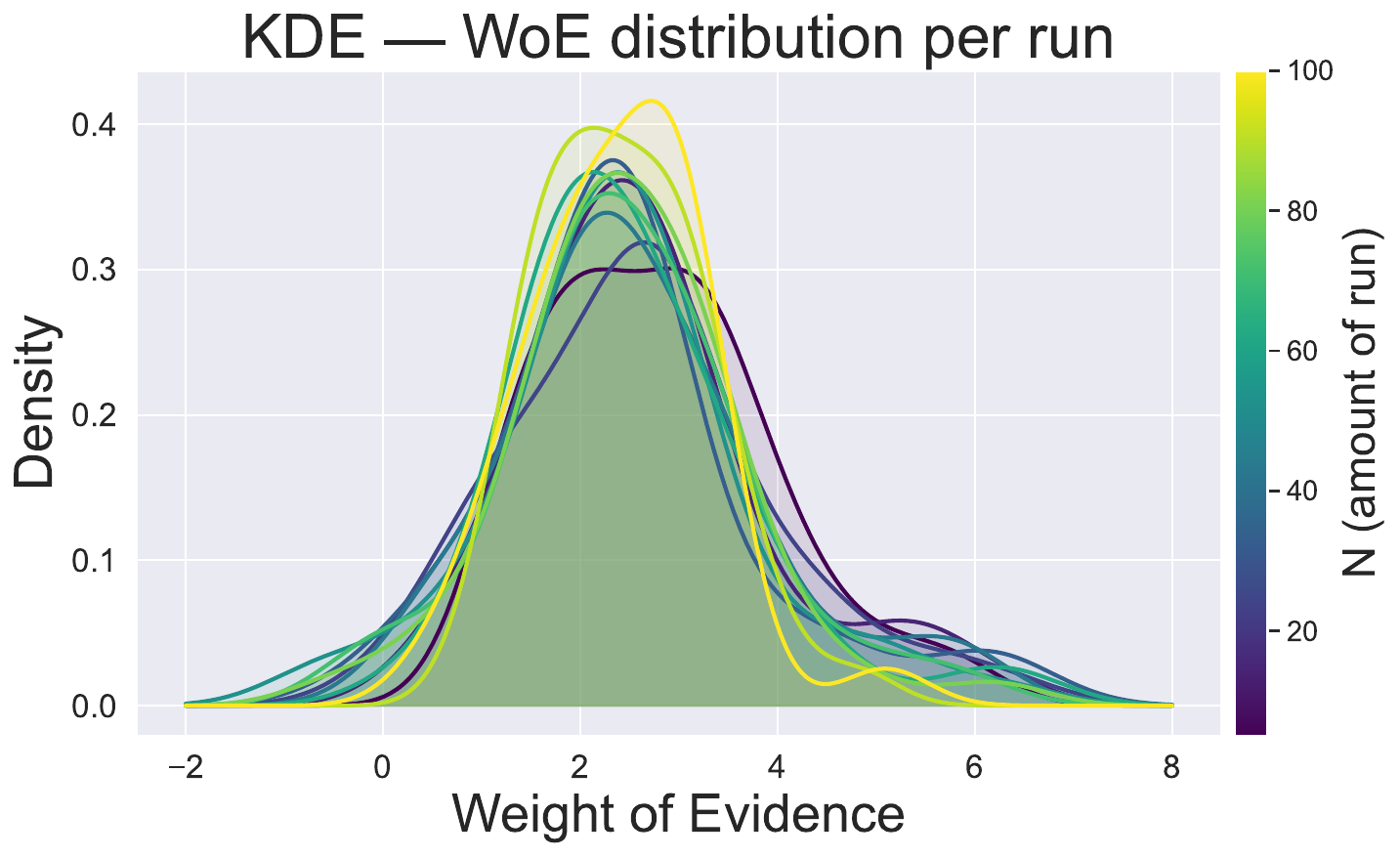}
    \caption{Gaussian KDE of WoE distributions across different numbers of runs $N$. The strong overlap between the distributions highlights the stability of WoE values.}
    \label{fig:woe_runs}
\end{minipage}\hfill
\begin{minipage}{0.48\columnwidth}
    \includegraphics[width=\linewidth]{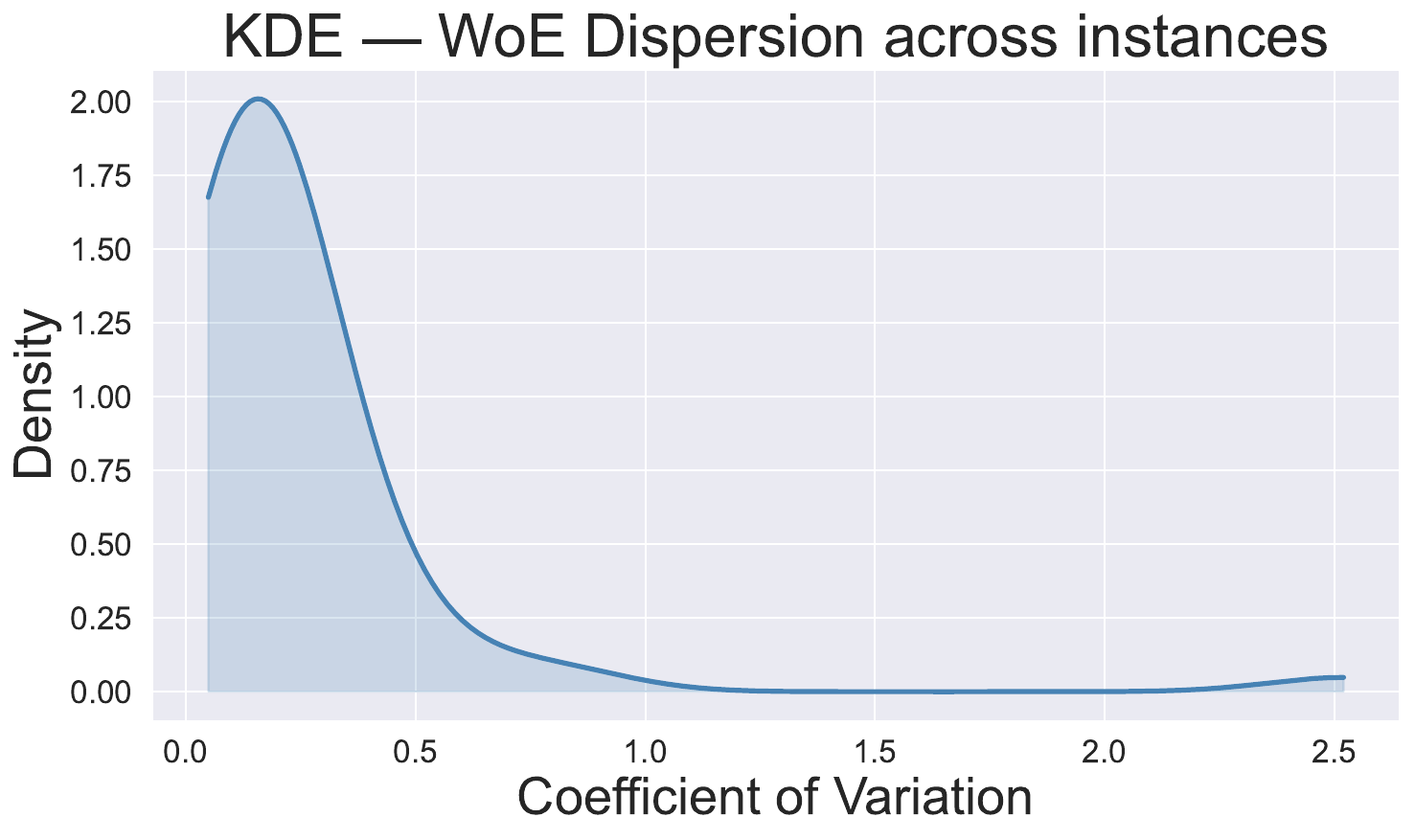}
    \caption{Gaussian KDE of the CV of WoE values across different numbers of runs for 50 instances. The plot is concentrated around values close to $0$, indicating consistency of WoE estimates.}
    \label{fig:woe_cv}
\end{minipage}
\end{figure}

\section{Limitations}\label{sec:limitations}
The flexibility of the proposed framework comes at the cost of several design choices. In this work, we adopted intuitive choices without claiming their optimality. For instance, as no prior knowledge was available, we modeled the prior probability as a uniform distribution over hypotheses. However, alternative priors could be naturally incorporated \cite{lunn2013bugs}. Similarly, the posterior distribution is estimated via the empirical frequency of a reference hypothesis across multiple runs, requiring exact equality between feature sets. This choice might be too restrictive. Softer alternatives could instead consider partial matches (e.g., $h^{(j)} \subseteq h^*$). More generally, a continuous formulation based on a distance measure $||h^{(j)} - h^*||$ could provide a smoother notion of similarity.
Another degree of freedom lies in the hypothesis generation process from feature importance scores. In this work, we opted for the TK and AS strategies but alternative options could be explored. Finally, the proposed approach requires running the explanatory method multiple times (even though we have shown that stable estimates of WoE can be obtained with just a few runs), making the overall procedure dependent on the computational cost of the FIM.

\section{Conclusion and Future Work}
\label{sec:conclusion}
In this work, we proposed a framework that embeds FIMs within a hypothesis-testing perspective based on WoE. 
The key design choice is the source of the reference
hypothesis $h^*$, which determines what the WoE is actually measuring:
alignment with domain knowledge, faithfulness to the intrinsic mechanism or internal stability of the FIM across repeated runs. In short, the framework not only allows one to evaluate a hypothesis against an alternative, but also to quantify key properties within the evaluation landscape of FIMs. 
On the theoretical side, we established a formal connection between the stability of a FIM and the WoE.
The experimental setup led to several insights: WoE detects the local-global discrepancy; it identifies disalignment along the Data-Model-FIM pipeline; and it measures the internal consistency (variability, stability) of FIMs.

Future work will focus on relaxing some of the current design
choices, in particular by developing softer formulations of the posterior that allow partial hypothesis matches, and by exploring alternative prior distributions beyond the uniform assumption. We also plan to extend the evaluation to a broader range of FIMs, model types, and data modalities. Finally, we believe that the experiment described in Section \ref{sec:GT_knowledge} represents a promising avenue to explore, particularly with regards to the concept of faithfulness, which has always been studied exclusively in relation to the model while ignoring the potential disalignment between the model itself and the underlying data generating process.

\section*{Acknowledgements}
Eddie Conti has been partially supported by the predoctoral grant FI-STEP (2025 STEP 00108) from the Research and University Department of the Generalitat de Catalunya and cofunded by the European Social Fund Plus. \'Alvaro Parafita acknowledges his AI4Science fellowship within the “Generacion D” initiative by Red.es, Ministerio para la Transformación Digital y de la Función Pública, for talent attraction (C005/24-ED CV1), funded by NextGenerationEU through PRTR. Axel Brando received funding from the Horizon Europe Programme under the AI4DEBUNK Project (https://www.ai4debunk.eu), grant agreement num. 101135757.
Work partially supported by the Italian Ministry of Education and Research (MUR) in the framework of the FoReLab project (Departments of Excellence).

%
%
\bibliographystyle{splncs04}
\bibliography{mybib}

@article{lundberg2017unified,
  title={A unified approach to interpreting model predictions},
  author={Lundberg, Scott M and Lee, Su-In},
  journal={Advances in neural information processing systems},
  volume={30},
  year={2017}
}

@inproceedings{LIME,
author = {Ribeiro, Marco Tulio and Singh, Sameer and Guestrin, Carlos},
title = {"Why Should I Trust You?": Explaining the Predictions of Any Classifier},
year = {2016},
isbn = {9781450342322},
publisher = {Association for Computing Machinery},
booktitle = {Proceedings of the 22nd ACM SIGKDD International Conference on Knowledge Discovery and Data Mining},
}

@inproceedings{melis2021human,
  title={From human explanation to model interpretability: A framework based on weight of evidence},
  author={Melis, David Alvarez and Kaur, Harmanpreet and Daum{\'e} III, Hal and Wallach, Hanna and Vaughan, Jennifer Wortman},
  booktitle={Proceedings of the AAAI Conference on Human Computation and Crowdsourcing},
  volume={9},
  year={2021}
}

@inproceedings{alshehri2023explainable,
  title={Explainable goal recognition: a framework based on weight of evidence},
  author={Alshehri, Abeer and Miller, Tim and Vered, Mor},
  booktitle={Proceedings of the International Conference on Automated Planning and Scheduling},
  volume={33},
  year={2023}
}

@misc{le2025evidencedecisionexploringevaluative,
      title={From Evidence to Decision: Exploring Evaluative AI}, 
      author={Thao Le and Tim Miller and Liz Sonenberg and Ronal Singh and H. Peter Soyer},
      year={2025},
      eprint={2402.01292},
      archivePrefix={arXiv},
}

@article{wod1985weight,
  title={Weight of evidence: A brief survey},
  author={Wod, IJ},
  journal={Bayesian statistics},
  volume={2},
  year={1985}
}

@inproceedings{alshehri2021human,
  title={Human centered explanation for goal recognition system.},
  author={Alshehri, Abeer and Miller, Tim and Vered, Mor and Alamri, Hajar},
  booktitle={IJCAI-PRICAI Workshop On Explainable Artificial Intelligence (XAI) 2020},
  year={2021},
  organization={Association for the Advancement of Artificial Intelligence (AAAI)}
}

@inproceedings{stability_perturbation,
  title={Towards reliable explainable AI: a novel stability metric for trustworthy interpretations},
  author={Butt, Talal Ashraf and Iqbal, Muhammad},
  booktitle={International Conference on Multimedia Systems and Signal Processing},
  pages={198--210},
  year={2025},
  organization={Springer}
}

@article{miller2019explanation,
  title={Explanation in artificial intelligence: Insights from the social sciences},
  author={Miller, Tim},
  journal={Artificial intelligence},
  volume={267},
  year={2019},
  publisher={Elsevier}
}

@article{lipton2018mythos,
  title={The mythos of model interpretability: In machine learning, the concept of interpretability is both important and slippery.},
  author={Lipton, Zachary C},
  journal={Queue},
  volume={16},
  number={3},
  year={2018},
  publisher={ACM New York, NY, USA}
}

@inproceedings{jacovi2021contrastive,
  title={Contrastive Explanations for Model Interpretability},
  author={Jacovi, Alon and Swayamdipta, Swabha and Ravfogel, Shauli and Elazar, Yanai and Choi, Yejin and Goldberg, Yoav},
  booktitle={Proceedings of the 2021 Conference on Empirical Methods in Natural Language Processing},
  year={2021}
}

@inproceedings{kadir2023evaluation,
  title={Evaluation metrics for xai: A review, taxonomy, and practical applications},
  author={Kadir, Md Abdul and Mosavi, Amir and Sonntag, Daniel},
  booktitle={2023 IEEE 27th International Conference on Intelligent Engineering Systems (INES)},
  year={2023},
  organization={IEEE}
}

@article{rudin2023globally,
  title={Globally-consistent rule-based summary-explanations for machine learning models: application to credit-risk evaluation},
  author={Rudin, Cynthia and Shaposhnik, Yaron},
  journal={Journal of Machine Learning Research},
  volume={24},
  number={16},
  year={2023}
}

@article{geng2022computing,
  title={Computing rule-based explanations by leveraging counterfactuals},
  author={Geng, Zixuan and Schleich, Maximilian and Suciu, Dan},
  journal={arXiv preprint arXiv:2210.17071},
  year={2022}
}

@article{lunn2013bugs,
  title={The BUGS book},
  author={Lunn, David and Jackson, Christopher and Best, Nicky and Thomas, Andrew and Spiegelhalter, David},
  journal={A practical introduction to Bayesian analysis, Chapman Hall, London},
  year={2013}
}

@misc{heart_disease,
  author       = {Janosi, Andras and Steinbrunn, William and Pfisterer, Matthias and Detrano, Robert},
  title        = {{Heart Disease}},
  year         = {1989},
  howpublished = {UCI Machine Learning Repository},
}

@misc{diabetes_dataset,
  author       = {Vincent Sigillito},
  title        = {Diabetes Dataset},
  year         = {1990},
  note         = {Research Center, RMI Group Leader, Applied Physics Laboratory, The Johns Hopkins University, Laurel, MD, USA},
}

@inproceedings{Mobile_price,
  title={Classification of mobile phone price dataset using machine learning algorithms},
  author={Hu, Ningyuan},
  booktitle={2022 3rd International Conference on Pattern Recognition and Machine Learning (PRML)},
  year={2022},
  organization={IEEE}
}

@inproceedings{churn_dataset,
  title={Predictive Analytics for Customer Retention: A CatBoost Model for Churn Detection},
  author={Sharma, Neha and Awasthi, Anmol and Parmar, Devendra Singh and Chouhan, Tejeshwari and Singh, Alka and Rawat, Daksh},
  booktitle={2025 International Conference on Networks and Cryptology (NETCRYPT)},
  year={2025},
  organization={IEEE}
}

@inproceedings{nayebi2023empirical,
  title={An empirical comparison of explainable artificial intelligence methods for clinical data: a case study on traumatic brain injury},
  author={Nayebi, Amin and Tipirneni, Sindhu and Foreman, Brandon and Reddy, Chandan K and Subbian, Vignesh},
  booktitle={AMIA annual symposium proceedings},
  volume={2022},
  year={2023}
}

@article{linardatos2020explainable,
  title={Explainable ai: A review of machine learning interpretability methods},
  author={Linardatos, Pantelis and Papastefanopoulos, Vasilis and Kotsiantis, Sotiris},
  journal={Entropy},
  year={2020},
  publisher={MDPI}
}

@article{islam2021explainable,
  title={Explainable artificial intelligence approaches: A survey},
  author={Islam, Sheikh Rabiul and Eberle, William and Ghafoor, Sheikh Khaled and Ahmed, Mohiuddin},
  journal={arXiv preprint arXiv:2101.09429},
  year={2021}
}

@article{ethics_ai2,
title = {Transparency and explainability of {AI} systems: From ethical guidelines to requirements},
journal = {Information and Software Technology},
year = {2023},
author = {Nagadivya Balasubramaniam and Marjo Kauppinen and Antti Rannisto and Kari Hiekkanen and Sari Kujala},
}

@article{intrinsic_posthoc,
author = {Du, Mengnan and Liu, Ninghao and Hu, Xia},
title = {Techniques for interpretable machine learning},
year = {2019},
publisher = {Association for Computing Machinery},
journal = {Commun. ACM},
}

@article{esteva2019guide,
  title={A guide to deep learning in healthcare},
  author={Esteva, Andre and Robicquet, Alexandre and Ramsundar, Bharath and Kuleshov, Volodymyr and DePristo, Mark and Chou, Katherine and Cui, Claire and Corrado, Greg and Thrun, Sebastian and Dean, Jeff},
  journal={Nature medicine},
  volume={25},
  number={1},
  year={2019},
  publisher={Nature Publishing Group US New York}
}

@article{cukierski2012titanic,
  title={Titanic-machine learning from disaster},
  author={Cukierski, Will},
  journal={Kaggle. available at: https://kaggle. com/competitions/titanic},
  year={2012}
}

@article{agarwal2022openxai,
  title={Openxai: Towards a transparent evaluation of model explanations},
  author={Agarwal, Chirag and Krishna, Satyapriya and Saxena, Eshika and Pawelczyk, Martin and Johnson, Nari and Puri, Isha and Zitnik, Marinka and Lakkaraju, Himabindu},
  journal={Advances in neural information processing systems},
  volume={35},
  year={2022}
}

@article{parola2026human,
  title={Human-centered xai via a concept-informed prompt-based validation framework for saliency maps [ciprova]},
  author={Parola, Marco and Alfeo, Antonio Luca and Cimino, Mario GCA},
  journal={Image and Vision Computing},
  year={2026},
  publisher={Elsevier}
}

@conference{daml24,
author={Chutong Huang},
title={The Prediction and Feature Importance Investigation in Titanic Survival Prediction},
booktitle={Proceedings of the 2nd International Conference on Data Analysis and Machine Learning - Volume 1: DAML},
year={2024},
publisher={SciTePress},
}

@article{scikit-learn,
  title   = {Scikit-learn: Machine Learning in {P}ython},
  author  = {Pedregosa, F. and Varoquaux, G. and Gramfort, A. and
             Michel, V. and Thirion, B. and Grisel, O. and
             Blondel, M. and Prettenhofer, P. and Weiss, R. and
             Dubourg, V. and Vanderplas, J. and Passos, A. and
             Cournapeau, D. and Brucher, M. and Perrot, M. and
             Duchesnay, E.},
  journal = {Journal of Machine Learning Research},
  volume  = {12},
  year    = {2011}
}

\end{document}